\documentclass[11pt]{article}

\usepackage[margin=1in]{geometry}

\usepackage{amsmath,amssymb,amsthm,mathtools}

\usepackage[numbers]{natbib}
\usepackage{hyperref}
\usepackage{cleveref}

\usepackage{graphicx}
\usepackage{booktabs}
\usepackage{tikz}
\usetikzlibrary{arrows.meta,positioning,decorations.pathmorphing,decorations.pathreplacing,calc}

\newtheorem{theorem}{Theorem}
\newtheorem{proposition}[theorem]{Proposition}

\newtheorem{corollary}[theorem]{Corollary}
\newtheorem{definition}[theorem]{Definition}

\DeclareMathOperator{\Tr}{Tr}

\DeclareMathOperator{\rank}{rank}
\DeclareMathOperator{\kernel}{ker}

\DeclareMathOperator{\Stab}{Stab}

\newcommand{\CC}{\mathbb{C}}
\newcommand{\RR}{\mathbb{R}}
\newcommand{\ZZ}{\mathbb{Z}}

\newcommand{\fhat}{\hat{f}}
\newcommand{\dmu}{\,d\mu}

\title{What Your Model Threw Away and Why You'll Want It Back: Masking, Fingerprinting, and Privacy from Discarded Geometry}

\author{
Zachary P.\ Bradshaw \\
QodeX Quantum, Inc.\ \\
Chicago, IL \\
\texttt{zak@qodexquantum.ai}
}

\date{}

\begin{document}

\maketitle

\begin{abstract}
We develop a framework for the information discarded by machine learning models whose inputs carry a Lie group action. Given a representation $\pi$ of a Lie group $G$ on a space $V$ and a learned function $f\colon V \to \RR$, we define two objects measuring the symmetry invisible to $f$. The \emph{null fiber} at a point $x \in V$ is the set $N_G(f,x) = \{g \in G : f(\pi(g^{-1}) \cdot x) = f(x)\}$ of group elements whose inverse action on $x$ is undetectable by $f$. When $N_G(f,x)$ is independent of $x$, it coincides with the \emph{stabilizer} $\Stab_G(f)$, the largest subgroup of $G$ under which $f$ is invariant. For smooth maps to $\RR$, the preimage theorem guarantees that null fibers have dimension at least $\dim G - 1$ at generic inputs, regardless of architecture. For compact groups acting on themselves, the Peter--Weyl theorem yields a spectral characterization of both objects in terms of the Fourier coefficient matrices of $f$. We show that null fiber elements can be computed efficiently via Newton iteration on the orbit map, at a cost comparable to a few gradient evaluations. Applications to data masking, model fingerprinting, and privacy-preserving computation are developed and tested experimentally on molecular property prediction under $\mathrm{SO}(3)$ and spherical image classification under the M\"obius group $\mathrm{PSL}(2, \CC)$. The framework applies uniformly to classical neural networks and variational quantum circuits.
\end{abstract}

\section{Introduction}\label{sec:intro}

Most machine learning models discard information about their input, and the set of input perturbations invisible to the model serves as a structural signature of what the model has learned to ignore. For neural networks on $\RR^n$, this idea was introduced by Cook, Zare, and Gader~\cite{cook2025competency,cook2020outlier}, who analyzed the null space projections of individual layer weight matrices as a tool for out-of-distribution detection and model interpretability. Li and Short~\cite{li2025null} subsequently studied the null space of the full input--output map $f\colon \RR^n \to \RR^m$, defined as the set of directions $v \in \RR^n$ satisfying $f(x + av) = f(x)$ for all inputs $x$ and scalars $a$. For ReLU networks, this null space equals the kernel of the first weight matrix and can be enormous, often comparable in dimension to the input space itself. Li and Short exploited this to demonstrate a striking form of image steganography in which replacing the null space component of an image produces a visually unrecognizable image that the network classifies identically. Viewed in the lens of secure communication, the trained network serves, without modification, as a decryption key.

These results rely on the additive structure of $\RR^n$. The perturbation $x \mapsto x + av$ is the action of the translation group, and the null space is a closed subgroup of $(\RR^n, +)$ that inherits the linear structure of $\RR^n$. More broadly, however, many machine learning models process data that transforms under a group representation~\cite{hall2015lie,steinberg2012representation}. Point clouds and molecular geometries carry representations of $\mathrm{SO}(3)$~\cite{thomas2018tensor}; rigid body configurations carry representations of $\mathrm{SE}(3)$~\cite{fuchs2020se}; images transform under the Euclidean group; and gauge field configurations transform under $\mathrm{SU}(N)$~\cite{boyda2021sampling,kanwar2020equivariant}. A growing body of work in geometric deep learning~\cite{bronstein2021geometric,bronstein2017geometric} and group-equivariant networks~\cite{bradshaw2025learning,cohen2016group,gerken2023geometric,kondor2018generalization} builds models that respect these group actions, motivating a generalization of null space theory to the representation-theoretic setting.

In this paper we carry out this generalization, and in doing so identify a conceptual distinction that the Euclidean case obscures. There are two related notions of ``information discarded by a model''. The \emph{stabilizer} $\Stab_G(f) = \{g \in G : f(\pi(g^{-1}) \cdot x) = f(x) \text{ for all } x\}$ under the left regular action induced by $\pi$ captures global invariance, consisting of the group elements invisible to $f$ regardless of input. The \emph{null fiber} $N_G(f, x) = \{g \in G : f(\pi(g^{-1}) \cdot x) = f(x)\}$ captures pointwise invariance, consisting of the group elements invisible at a specific input $x$. The stabilizer is always contained in every null fiber, since $\Stab_G(f) = \bigcap_x N_G(f,x)$, but may be vastly smaller. This distinction matters because the stabilizer is often trivial. The null fiber, by contrast, is generically large. The Euclidean case is special because the two notions coincide for the function class of interest. For a ReLU network with first weight matrix $W_1$, the null fiber at every input is $N_G(f, x) = \kernel(W_1)$, independent of $x$. Since all fibers are the same subgroup of $(\RR^n, +)$, the stabilizer $\Stab_G(f) = \bigcap_x N_G(f, x) = \kernel(W_1)$ equals every fiber. This collapse of the global and pointwise notions is a consequence of the linearity of the first layer and the additive structure of $\RR^n$, and does not persist for nonlinear models on nonabelian groups.

The null fiber and stabilizer admit complementary uses. As measurements, $\Stab_G(f)$ and $N_G(f,x)$ quantify the symmetry a model has learned, the global invariance it enforces and the pointwise invariance it exhibits at each input, which is informative when the true symmetry of the data is only approximate and one wishes to recover it from a trained model. On the other hand, the same invariance is a license to alter the input without changing the output, and it is this latter use case that we develop here.

Our main results are a spectral characterization of both the null fiber and stabilizer for functions on compact groups via the Peter--Weyl theorem, a gradient-based algorithm that computes exact null fiber elements in $O(1)$ model evaluations via Newton iteration on the orbit map, and applications to data masking, model fingerprinting, and privacy-preserving computation that do not require the stabilizer to be nontrivial. We demonstrate the framework experimentally on molecular property prediction under $\mathrm{SO}(3)$ and spherical image classification under the M\"obius group $\mathrm{PSL}(2, \CC)$, where null fiber masking renders spherical images unrecognizable while preserving the classifier's prediction. The theory applies uniformly to classical neural networks and variational quantum circuits~\cite{biamonte2017quantum,cerezo2021variational,gili2023quantum,schuld2020circuit}.

\section{Theory}\label{sec:null-fibers-subgroups}

Let $G$ be a Lie group and $\pi\colon G \to \mathrm{GL}(V)$ a continuous representation on a finite-dimensional space $V$. Throughout, $V$ is the continuous representation space on which the group acts smoothly, for instance the space of atomic coordinates or of band-limited spherical harmonic coefficients. A discrete rendering of the input, such as a pixel grid, is an input or output format and not the space on which the analysis is carried out, since the group action and the orbit map are defined on $V$. Let $f$ be a continuous function on $V$ learned by a machine learning model. The models we consider output real values, so we take $f\colon V \to \RR^m$. Where the harmonic analysis below requires it, we allow complex-valued $f\colon V \to \CC$; since $\CC \cong \RR^2$, this changes only the relevant codimension, replacing $m$ by $2m$.

\begin{definition}[Null fiber]\label{def:null-fiber}
The \emph{null fiber} of $f$ at $x \in V$ with respect to the representation $\pi$ is
\begin{equation}\label{eq:null-fiber}
    N_G(f, x) = \{ g \in G : f(\pi(g^{-1}) \cdot x) = f(x) \}.
\end{equation}
\end{definition}
\noindent The null fiber is the set of $g\in G$ for which the left regular action $(L_gf)(x)=f(\pi(g^{-1})\cdot x)$ acts trivially at $x$. The intersection over all $x\in V$ produces the subgroup of all $g\in G$ that fix the function $f$ under the left regular action, which is known as the stabilizer of $f$.
\begin{definition}[Stabilizer]\label{def:stabilizer}
The \emph{stabilizer} of $f$ with respect to $\pi$ is
\begin{equation}\label{eq:stabilizer}
    \Stab_G(f) = \{ g \in G : f(\pi(g^{-1}) \cdot x) = f(x) \text{ for all } x \in V \} = \bigcap_{x \in V} N_G(f, x).
\end{equation}
\end{definition}
\noindent The stabilizer and null fiber satisfy several algebraic and geometric properties, which we now show.
\begin{proposition}\label{prop:basic-properties}
Let $f\colon V \to \RR^m$ be continuous. Then
\begin{enumerate}
    \item $\Stab_G(f)$ is a closed subgroup of $G$,
    \item For each $x \in V$, the null fiber $N_G(f, x)$ is a closed subset of $G$ containing $\Stab_G(f)$,
    \item If $f$ is smooth and $e$ is a regular point of the orbit map $\phi_x(g)=f(\pi(g^{-1}) \cdot x)$, then in a neighborhood of $e$, $N_G(f, x)$ is a smooth submanifold of $G$ of dimension $\dim(G) - m$.
\end{enumerate}
\end{proposition}

\begin{proof}
(2) Define the continuous map $\phi_x(g)=f(\pi(g^{-1})x)$. Then $N_G(f,x)=\phi_x^{-1}(f(x))$, which is the preimage of a closed set (a point) under a continuous map and therefore closed.

(1) To see that the stabilizer is a subgroup of $G$, observe that if $g,h\in\Stab_G(f)$, then $f(\pi((gh)^{-1})x)=f(\pi(h^{-1}g^{-1})x)=f(\pi(h^{-1})\pi(g^{-1})x)=f(\pi(g^{-1})x)=f(x)$ for all $x\in V$. Moreover, $f(\pi(g)x)=f(\pi(g^{-1})\pi(g)x)=f(x)$, showing that the stabilizer is closed under inversion. To see that the stabilizer is topologically closed, just observe that each null fiber is closed by (2) and the stabilizer is an intersection of null fibers.

(3) The map $\phi_x\colon G \to \RR^m$ defined by $\phi_x(g) = f(\pi(g^{-1}) \cdot x)$ is smooth when $f$ and $\pi$ are, and $N_G(f, x) = \phi_x^{-1}(\phi_x(e))$. If $e$ is a regular point, then $d\phi_x|_e\colon T_e G \to \RR^m$ is surjective, so by the constant rank theorem $\phi_x$ has rank $m$ on a neighborhood $U$ of $e$. The preimage theorem then realizes $\phi_x^{-1}(\phi_x(e)) \cap U$ as an embedded submanifold of $G$ of codimension $m$, i.e.\ of dimension $\dim G - m$.
\end{proof}

The stabilizer is not necessarily a normal subgroup of $G$, though it is automatically normal when $G$ is abelian. An example where normality fails is given by $\mathrm{SO}(3)$ acting on $\RR^3$ with $f$ invariant only under rotations about a fixed axis. The stabilizer in this case is $\mathrm{SO}(2)_{\hat{n}}$, and conjugating by $R \in \mathrm{SO}(3)$ sends $\mathrm{SO}(2)_{\hat{n}}$ to $\mathrm{SO}(2)_{R\hat{n}}$, which is a different subgroup whenever $R$ moves $\hat{n}$.

The dimension bound in Proposition~\ref{prop:basic-properties} tells us that a model $f\colon V \to \RR$ on a $d$-dimensional representation has null fibers of dimension at least $\dim G - 1$ at generic points. This is a counting argument, not an architectural assumption. It holds for every smooth model and guarantees that information is always lost pointwise, even when the stabilizer is trivial.

As noted in the introduction, the distinction between null fibers and stabilizers is concealed in the euclidean case. For $G = (\RR^n, +)$ acting on $V = \RR^n$ by translation, a linear map $f(x) = Wx$ has null fibers $N_G(f, x) = \kernel(W)$ for every $x$, independent of the basepoint. The null fiber is the stabilizer, and both equal the classical null space. This independence of basepoint follows from the fact that translation commutes with linear maps, since $W(x + v) = Wx + Wv$, so $f(x + v) = f(x)$ if and only if $Wv = 0$, regardless of $x$. For nonlinear $f$ or nonabelian $G$, the null fibers generically vary with $x$, and the stabilizer, their intersection, is smaller and often trivial.

\subsection{Fourier characterization on compact groups}\label{sec:fourier}

We now specialize to the case $V = G$, where a compact Lie group acts on itself by left multiplication. Let $\hat{G}$ denote a complete set of representatives of equivalence classes of irreducible unitary representations $\rho_\lambda\colon G \to U(d_\lambda)$. By the Peter--Weyl theorem, any $f \in L^2(G)$ admits the Fourier expansion
\begin{equation}\label{eq:peter-weyl}
    f(g) = \sum_{\lambda \in \hat{G}} d_\lambda \, \Tr\!\left( \fhat(\lambda) \rho_\lambda(g) \right),
\end{equation}
where the Fourier coefficient matrices are
\begin{equation}\label{eq:fourier-coefficients}
    \fhat(\lambda) = \int_G f(g)\, \rho_\lambda(g)^\dagger \dmu(g) \in \CC^{d_\lambda \times d_\lambda}.
\end{equation}
The left regular representation $(L_g f)(h) = f(g^{-1}h)$ acts on the Fourier coefficients by right multiplication,
\begin{equation}\label{eq:left-fourier}
    \widehat{L_g f}(\lambda) = \fhat(\lambda)\rho_\lambda^\dagger(g),
\end{equation}
and this intertwining of left translation with matrix multiplication is the mechanism that makes both the stabilizer and the null fibers computable.

\begin{theorem}[Fourier characterization]\label{thm:fourier-characterization}
Let $G$ be a compact Lie group and $f \in L^2(G)$ with Fourier support $\Lambda(f) = \{\lambda \in \hat{G} : \fhat(\lambda) \neq 0\}$. An element $g$ lies in $\Stab_G(f)$ if and only if
\begin{equation}\label{eq:stab-fourier}
    \fhat(\lambda)\rho_\lambda^\dagger(g) = \fhat(\lambda)
\end{equation}
for every $\lambda\in \Lambda(f)$. Equivalently,
\begin{equation}\label{eq:stab-intersection}
    \Stab_G(f) = \bigcap_{\lambda \in \Lambda(f)} \Stab\!\left( \mathrm{row}\!\left(\fhat(\lambda)\right);\, \rho_\lambda^\dagger \right),
\end{equation}
where $\Stab(W; \rho_\lambda^\dagger) = \{g \in G : w\rho_\lambda^\dagger(g) = w \text{ for all } w \in W\}$ and $\mathrm{row}(\fhat(\lambda))$ denotes the row space of $\fhat(\lambda)$. An element $g$ lies in $N_G(f, h_0)$ if and only if
\begin{equation}\label{eq:null-fiber-fourier}
    \sum_{\lambda \in \hat{G}} d_\lambda \, \Tr\!\left(\fhat(\lambda) [\rho_\lambda^\dagger(g) - I] \rho_\lambda(h_0) \right) = 0.
\end{equation}
\end{theorem}

\begin{proof}
For the stabilizer, by taking Fourier transforms, $L_g f = f$ in $L^2(G)$ if and only if $\fhat(\lambda)\rho_\lambda^\dagger(g) = \fhat(\lambda)$ for every $\lambda$. The condition $\fhat(\lambda)[\rho_\lambda^\dagger(g) - I] = 0$ means every row of $\fhat(\lambda)$ is annihilated by $\rho_\lambda^\dagger(g) - I$, giving~\eqref{eq:stab-intersection}. Finally, for the null fiber, $g \in N_G(f, h_0)$ if and only if $f(g^{-1} h_0) = f(h_0)$, which by~\eqref{eq:peter-weyl} is condition~\eqref{eq:null-fiber-fourier}.
\end{proof}

The asymmetry between the two characterizations explains why the stabilizer is generically much smaller than the null fibers. The stabilizer condition~\eqref{eq:stab-fourier} is a system of matrix equations, one for each irreducible representation, that must all hold simultaneously. The null fiber condition~\eqref{eq:null-fiber-fourier} is a single scalar equation depending on the basepoint $h_0$.

\begin{corollary}\label{cor:dimension}
If $\fhat(\lambda)$ has full rank $d_\lambda$ for some $\lambda$, then $\Stab_G(f) \subseteq \kernel(\rho_\lambda)$. If $\Lambda(f) = \hat{G}$ and $\fhat(\lambda)$ has full rank for all $\lambda$, then $\Stab_G(f) = \{e\}$.
\end{corollary}
\begin{proof}
Let $g\in\Stab_G(f)$. The stabilizer condition requires $\fhat(\lambda)\rho_\lambda^\dagger(g) = \fhat(\lambda)$ for every $\lambda \in \Lambda(f)$. The Fourier coefficient $\fhat(\lambda)$ is a $d_\lambda \times d_\lambda$ matrix, so full rank implies invertibility. Left-multiplying by $\fhat(\lambda)^{-1}$ gives $\rho_\lambda(g) = I$, so $g \in \kernel(\rho_\lambda)$. If $\fhat(\lambda)$ is invertible for every $\lambda \in \hat{G}$, then $g$ lies in the kernel of every irreducible representation, and by the Peter--Weyl theorem, the only such element is the identity.
\end{proof}

The stabilizer is nontrivial when $f$ has sparse Fourier support and low-rank Fourier coefficients, conditions that are forced by architectural constraints such as finite-width equivariant layers. Regardless of whether the stabilizer is trivial, the null fiber $N_G(f, h_0)$ has dimension at least $\dim G - 1$ for real-valued $f$ at every regular point.

When $G = (\RR^n, +)$ acts on itself by translation, the irreducible representations are one-dimensional characters $\rho_k(v) = e^{ik \cdot v}$ for $k \in \RR^n$, and the Fourier coefficients are scalars $\fhat(k) \in \CC$. The stabilizer condition of \Cref{thm:fourier-characterization} becomes $\fhat(k) e^{-ik \cdot v} = \fhat(k)$ for every $k \in \Lambda(f)$, which for $\fhat(k) \neq 0$ reduces to $e^{-ik \cdot v} = 1$, or equivalently $k \cdot v \in 2\pi\ZZ$. Thus $\Stab_G(f) = \{v\in\RR^n : k\cdot v\in 2\pi\ZZ \text{ for all } k\in\Lambda(f)\}$. Writing $S=\operatorname{span}\Lambda(f)$ and decomposing $v=v_\parallel+v_\perp$ with $v_\parallel\in S$ and $v_\perp\in S^\perp$, we have $k\cdot v=k\cdot v_\parallel$ for every $k\in\Lambda(f)$, so the condition constrains only $v_\parallel$ and leaves $v_\perp$ unconstrained. Hence
\begin{equation}
  \Stab_G(f) = L \oplus S^\perp,
  \qquad
  L = \{w\in S : k\cdot w\in 2\pi\ZZ \text{ for all } k\in\Lambda(f)\}.
\end{equation}
Since $\Lambda(f)$ spans $S$, the functionals $w\mapsto k\cdot w$ separate points of $S$, so $L$ is discrete, i.e.\ a lattice in $S$. This is the general form of a closed subgroup of $(\RR^n,+)$: a subspace $S^\perp$ together with a lattice $L$ in a complementary subspace. It is \emph{not} a vector subspace unless $L=\{0\}$. A single sinusoid $f(x)=\cos(k_0\cdot x)$ already shows this, with $\Stab_G(f) = \tfrac{2\pi}{\lvert k_0\rvert^2}\,k_0\ZZ \oplus k_0^\perp$, whose lattice part records the spatial period of $f$ along $k_0$.

The null space of \cite{li2025null} is $N(f)=\{v\in\RR^n : f(x+av)=f(x)\text{ for all } x\in\RR^n,\ a\in\RR\}$, which demands invariance along the entire line $\RR v$. Setting $a=1$ shows $N(f)\subseteq\Stab_G(f)$. For the converse direction, invariance for \emph{all} $a$ forces $a\,(k\cdot v)\in 2\pi\ZZ$ for every $a\in\RR$, which holds only if $k\cdot v=0$; thus $N(f) = (\operatorname{span}\Lambda(f))^\perp = S^\perp$. Therefore $N(f)$ is exactly the identity component of $\Stab_G(f)$, equivalently the maximal subspace it contains, and the inclusion $N(f)\subseteq\Stab_G(f)$ is strict precisely when the lattice part $L$ is nontrivial. This is generic for finitely supported spectra; the two coincide only when $L=\{0\}$.

\subsection{Computing Null Fibers}\label{sec:computation}

The applications described in Section~\ref{sec:applications} require a party holding sensitive data to find null fiber elements efficiently. We now show that, given gradient access to $f$, a null fiber element can be computed at the cost of a few gradient evaluations, regardless of the dimension of $G$.

Fix an input $x \in V$ and define the orbit map $\phi\colon G \to \RR$ by $\phi(g) = f(\pi(g^{-1}) \cdot x)$. The null fiber at $x$ is $N_G(f, x) = \phi^{-1}(\phi(e))$. We work in the Lie algebra $\mathfrak{g} = T_e G$ via the exponential map, defining $\psi\colon \mathfrak{g} \to \RR$ by
\begin{equation}\label{eq:psi}
    \psi(v) = \phi(\exp(v)) = f(\pi(\exp(-v)) \cdot x).
\end{equation}
Finding a null fiber element reduces to finding $v \in \mathfrak{g}$ with $\psi(v) = \psi(0)$. Since $\exp\colon\mathfrak{g}\to G$ need not be injective, a Lie algebra element $v$ specifies a single group element $\exp(v)$ but a given group element may have several preimages. Our algorithm produces null fiber elements rather than Lie algebra representatives, so this is immaterial. We seek a $g\in N_G(f,x)$, and any preimage realizing it suffices. For compact connected $G$, including $SO(3)$, $\exp$ is surjective, so every group element is reached by some $v\in\mathfrak{g}$, in particular every null fiber element at any rotation angle.

Choose a basis $\{E_1, \ldots, E_d\}$ for $\mathfrak{g}$, where $d = \dim G$, and write $v = v_1 E_1 + \cdots + v_d E_d$. The derivative of $\psi$ at the origin is the row vector $J \in \RR^{1 \times d}$ with entries
\begin{equation}\label{eq:jacobian-entry}
    J_i = \frac{\partial \psi}{\partial v_i}\bigg|_{v=0} = \lim_{t \to 0} \frac{f(\pi(\exp(-tE_i)) \cdot x) - f(x)}{t}.
\end{equation}
To first order, $\pi(\exp(-tE_i)) \cdot x = x - t\, d\pi(E_i) \cdot x + O(t^2)$, where $d\pi\colon \mathfrak{g} \to \mathfrak{gl}(V)$ is the derived representation defined by $d\pi(E) = \frac{d}{dt}\big|_{t=0} \pi(\exp(tE))$. Applying the chain rule gives
\begin{equation}\label{eq:jacobian-formula}
    J_i = -\nabla f(x)^T \left( d\pi(E_i) \cdot x \right),
\end{equation}
where $\nabla f(x) \in \RR^{\dim V}$ is the gradient of $f$ at $x$. The matrix $d\pi(E_i) \in \RR^{\dim V \times \dim V}$ depends only on the representation and the basis, not on the model. For $\mathrm{SO}(3)$ in the defining representation on $\RR^3$, each $d\pi(E_i)$ is a $3 \times 3$ skew-symmetric matrix, and $d\pi(E_i) \cdot x$ is a cross product. Computing $\nabla f(x)$ requires one backpropagation pass through $f$, which is the same operation used in training. Assembling the full Jacobian $J$ then costs $d$ matrix-vector products of the form $d\pi(E_i) \cdot x$, which are cheap linear algebra operations that do not involve the model.

The Jacobian $J$ is a $1 \times d$ matrix. At a regular point, $J \neq 0$, and the kernel $\ker(J) \subseteq \mathfrak{g}$ has dimension $d - 1$. This kernel is the tangent space to the null fiber at the identity. Any $v \in \ker(J)$ satisfies
\begin{equation}
    \psi(tv) = \psi(0) + O(t^2),
\end{equation}
so $\exp(tv)$ lies in the null fiber to first order. The $O(t^2)$ scaling, rather than $O(t)$, is the signature of a tangent direction; a vector outside $\ker(J)$ produces $\psi(tv) - \psi(0) = t \cdot Jv + O(t^2)$, which is $O(t)$. Computing $\ker(J)$ is a standard linear algebra operation. Since $J$ is a single row vector, $\ker(J)$ is the orthogonal complement of $J^T$ in $\RR^d$, a subspace of dimension $d - 1$.

Pick a vector $v_0 \in \ker(J)$ and set $g_0 = \exp(v_0)$, so that $g_0$ is in the null fiber to first order. The first-order approximation guarantees $\psi(v_0) - \psi(0) = O(\|v_0\|^2)$, but for large $\|v_0\|$ this error may be substantial. Newton's method refines this first-order guess into an exact null fiber element by repeatedly linearizing the orbit map and stepping toward its level set. The equation $\psi(v)-\psi(0)=0$ is a single scalar constraint on the $d=\dim G$ unknowns in $v$, so it is underdetermined; its solution set is the $(d-1)$-dimensional null fiber rather than an isolated point. The Newton step we use therefore does not solve a square system; it takes the minimum-norm correction that linearly cancels the current residual, which moves the iterate onto the level set without committing to a particular point on it. At the current iterate $v_k$, we compute $\psi(v_k)$ by a forward pass through $f$ and the derivative $J_{v_k} = \nabla_v \psi|_{v_k} \in \RR^{1 \times d}$ by a backward pass. The residual is $\delta_k = \psi(v_k) - \psi(0)$. The minimum-norm correction that kills the residual to first order is
\begin{equation}\label{eq:newton-step}
    \Delta v_k = -\delta_k \frac{J_{v_k}^T}{\|J_{v_k}\|^2},
\end{equation}
which is the pseudoinverse solution to $J_{v_k} \Delta v = -\delta_k$. Update $v_{k+1} = v_k + \Delta v_k$. By the standard Newton convergence argument, $\delta_{k+1} = O(\delta_k^2)$, so the number of correct digits doubles at each step. Starting from a residual of order $1$, four or five iterations typically reach machine precision. Each iteration requires one forward pass to evaluate $\psi(v_k)$ and one backward pass to compute $J_{v_k}$. When the model is highly nonlinear near $v_k$, the full Newton step may overshoot, and a backtracking line search restores convergence; accept the step if $|\delta_{k+1}| < |\delta_k|$, otherwise halve $\Delta v_k$ and retry.

The total cost to find one null fiber element is as follows. Computing $\nabla f(x)$ requires one backward pass through $f$. Assembling the Jacobian $J$ and its kernel requires $d$ matrix-vector products, where $d = \dim G$. The Newton iteration requires 3--5 steps of one forward and one backward pass each. The total is roughly $10$ model evaluations, independent of $\dim G$, since the Lie algebra operations are dominated by the model evaluations. To mask $N$ data points, the masking party repeats this procedure for each input, at a cost of roughly $10N$ model evaluations. The Jacobian cannot be reused across data points because $\nabla f(x)$ depends on $x$, but the Newton iterations for different inputs are independent and can be batched, reducing the wall-clock cost to roughly $10$ sequential model evaluations with $N$-way parallelism.

\section{Examples}\label{sec:examples}

\subsection{$\mathrm{SO}(3)$ in the defining representation on $\RR^3$}\label{ex:so3-action}

Let $\pi$ be the defining representation of $G = \mathrm{SO}(3)$ on $V = \RR^3$, and let $f\colon \RR^3 \to \RR$ be a learned function. The stabilizer $\Stab_G(f) = \{R \in \mathrm{SO}(3) : f(R^{-1}x) = f(x) \text{ for all } x\}$ consists of all rotations under which $f$ is globally invariant, and its structure reflects the symmetry of $f$. If $f$ depends only on $\|x\|$, then $f$ is invariant under all rotations and $\Stab_G(f) = \mathrm{SO}(3)$. If $f$ depends on $\|x\|$ and the angle $x$ makes with a fixed axis $\hat{n}$, then $\Stab_G(f) \cong \mathrm{SO}(2)_{\hat{n}}$, the group of rotations about $\hat{n}$. This subgroup is not normal in $\mathrm{SO}(3)$, since conjugating by $R$ sends $\mathrm{SO}(2)_{\hat{n}}$ to $\mathrm{SO}(2)_{R\hat{n}}$, which is a different subgroup whenever $R$ moves $\hat{n}$. If $f$ depends on all three components of $x$, then generically $\Stab_G(f) = \{I\}$.

The null fibers, however, are nontrivial regardless of the symmetry of $f$. At any fixed $x_0 \neq 0$, the map $g \mapsto f(\pi(g^{-1}) \cdot x_0)$ is a smooth function from the three-dimensional group $\mathrm{SO}(3)$ to $\RR$, and at a generic input the preimage theorem gives a null fiber of dimension $\dim(\mathrm{SO}(3)) - 1 = 2$. A model with trivial stabilizer still has two-dimensional null fibers at regular inputs; such a stabilizer says that no rotation is invisible everywhere, but the null fibers say that at each input, a two-parameter family of rotations is invisible there.

\subsection{Bandlimited functions on $\mathrm{SO}(3)$}\label{ex:so3}

Now let $G = \mathrm{SO}(3)$ act on itself by left multiplication, so that the Fourier characterization of \Cref{thm:fourier-characterization} applies directly. The irreducible representations are the Wigner $D$-matrices $D^\ell$ of dimension $2\ell + 1$, indexed by $\ell \in \{0, 1, 2, \ldots\}$. Consider a model $f\colon \mathrm{SO}(3) \to \RR$ with Fourier support $\Lambda(f) = \{0, 1\}$. The $\ell = 0$ coefficient is a scalar, and the stabilizer condition $\fhat(0)\rho_0^\dagger(g) = \fhat(0)$ is satisfied for every $g$, so this coefficient imposes no constraint. The $\ell = 1$ representation is the defining representation on $\RR^3$, and $\fhat(1)$ is a $3 \times 3$ matrix whose rank determines the stabilizer.

If $\rank(\fhat(1)) = 1$, then the row space $\mathrm{row}(\fhat(1))$ is a line through the origin in $\RR^3$. The stabilizer condition $\fhat(1)\rho_1^\dagger(g) = \fhat(1)$ requires that $\rho_1^\dagger(g)$ fix every vector in this line, and the set of rotations fixing a line is the group of rotations about that line, isomorphic to $\mathrm{SO}(2)$. Thus $\Stab_G(f) \cong \mathrm{SO}(2)$, and the model discards one of three rotational degrees of freedom. If $\rank(\fhat(1)) = 2$, then $\mathrm{row}(\fhat(1))$ is a plane, and a rotation fixing every vector in a plane in $\RR^3$ must be the identity, since the only rotation whose $+1$ eigenspace contains a two-dimensional subspace is the trivial rotation. Thus $\Stab_G(f) = \{e\}$, and a rank-2 Fourier coefficient at $\ell = 1$ alone suffices to make the stabilizer trivial. The null fibers, however, remain two-dimensional at generic points in both cases, since the single scalar equation~\eqref{eq:null-fiber-fourier} imposes one real constraint on the three-dimensional group $\mathrm{SO}(3)$. The two cases are illustrated in \Cref{fig:so3}.

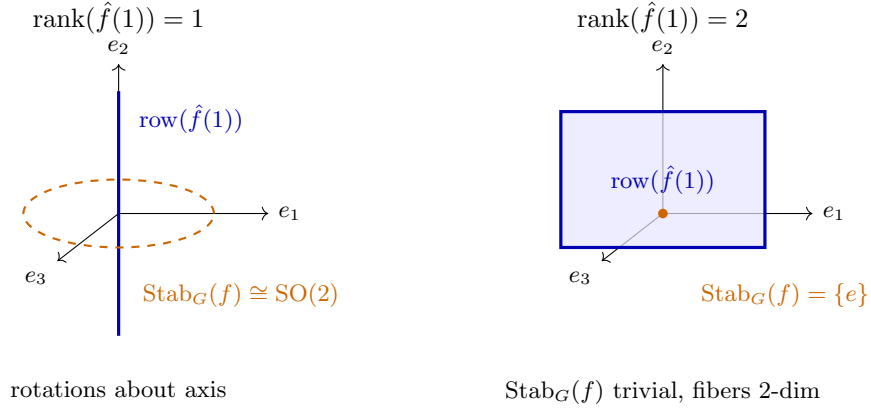
\begin{figure}[t]
\centering
\begin{tikzpicture}[scale=0.9, every node/.style={font=\small}]
% Case 1: rank 1
\begin{scope}[shift={(-4,0)}]
    \node[above, font=\bfseries\small] at (0,2.5) {$\rank(\fhat(1)) = 1$};
    \draw[->] (0,0) -- (2.2,0) node[right] {\footnotesize $e_1$};
    \draw[->] (0,0) -- (0,2.2) node[above] {\footnotesize $e_2$};
    \draw[->] (0,0) -- (-0.9,-0.7) node[below left] {\footnotesize $e_3$};
    \draw[very thick, blue!70!black] (0,-1.8) -- (0,1.8);
    \node[right, blue!70!black, font=\footnotesize] at (0.15,1.4) {$\mathrm{row}(\fhat(1))$};
    \draw[thick, orange!80!black, dashed] (0,0) ellipse (1.4 and 0.5);
    \node[orange!80!black, font=\footnotesize] at (1.8,-1.2) {$\Stab_G(f) \cong \mathrm{SO}(2)$};
    \node[below, font=\footnotesize] at (0,-2.3) {rotations about axis};
\end{scope}
% Case 2: rank 2
\begin{scope}[shift={(4,0)}]
    \node[above, font=\bfseries\small] at (0,2.5) {$\rank(\fhat(1)) = 2$};
    \draw[->] (0,0) -- (2.2,0) node[right] {\footnotesize $e_1$};
    \draw[->] (0,0) -- (0,2.2) node[above] {\footnotesize $e_2$};
    \draw[->] (0,0) -- (-0.9,-0.7) node[below left] {\footnotesize $e_3$};
    \fill[blue!10, opacity=0.7] (-1.5,-0.5) -- (1.5,-0.5) -- (1.5,1.5) -- (-1.5,1.5) -- cycle;
    \draw[very thick, blue!70!black] (-1.5,-0.5) -- (1.5,-0.5) -- (1.5,1.5) -- (-1.5,1.5) -- cycle;
    \node[blue!70!black, font=\footnotesize] at (0,0.5) {$\mathrm{row}(\fhat(1))$};
    \fill[orange!80!black] (0,0) circle (2pt);
    \node[orange!80!black, font=\footnotesize] at (1.8,-1.2) {$\Stab_G(f) = \{e\}$};
    \node[below, font=\footnotesize] at (0,-2.3) {$\Stab_G(f)$ trivial, fibers $2$-dim};
\end{scope}
\end{tikzpicture}
\caption{Stabilizer of a bandlimited function on $\mathrm{SO}(3)$ with Fourier support $\Lambda(f) = \{0,1\}$, as described in \Cref{ex:so3}. Left: $\fhat(1)$ has rank $1$, giving $\Stab_G(f) \cong \mathrm{SO}(2)$. Right: $\fhat(1)$ has rank $2$, giving $\Stab_G(f) = \{e\}$. In both cases, the null fibers are two-dimensional at generic points.}
\label{fig:so3}
\end{figure}

\subsection{$\mathrm{SE}(3)$ on rigid body configurations}\label{ex:se3}

Let $G = \mathrm{SE}(3) = \RR^3 \rtimes \mathrm{SO}(3)$ act on itself by left multiplication. Since $\mathrm{SE}(3)$ is noncompact, the Peter--Weyl theorem does not directly apply, but the stabilizer $\Stab_G(f)$ and null fibers $N_G(f, h_0)$ are well-defined as closed subsets of $G$. The stabilizer detects which rigid body transformations are invisible to the model: if $\Stab_G(f)$ contains the translation subgroup $\RR^3$, the model is insensitive to position, and if it contains a copy of $\mathrm{SO}(2)$ about some axis, the model is insensitive to rotations about that axis. For a sufficiently expressive model, $\Stab_G(f) = \{e\}$ generically, but even then the null fibers are large. The group $\mathrm{SE}(3)$ is six-dimensional, so for a real-valued model, the null fibers have dimension at least five at generic inputs, providing five dimensions of pointwise masking even when no global invariance exists.

\section{Applications}\label{sec:applications}

The stabilizer provides an algebraically exact certificate of invariance, but its utility depends on it being nontrivial. The null fiber framework resolves this limitation. In this section, we describe three applications where the distinction between the stabilizer and null fibers is practically significant.

\subsection{Data Masking}\label{sec:encryption}

The stabilizer endows every model with a global data masking scheme, but its utility is limited by the requirement that $\Stab_G(f)$ be nontrivial. Null fibers provide a strictly more general pointwise scheme that works for every model.

\begin{proposition}[Model-as-decoder]\label{prop:auto-decrypt}
Let $\pi$ be a representation of $G$ on $V$, and let $f\colon V \to \CC$. For any $x \in V$ and any $g \in N_G(f, x)$,
\begin{equation}
    f(\pi(g^{-1}) \cdot x) = f(x).
\end{equation}
In particular, if $g \in \Stab_G(f)$, the same identity holds for all $x$ simultaneously.
\end{proposition}

The global protocol proceeds as follows. Alice and Bob share a trained model $f$ and the stabilizer $\Stab_G(f)$. Alice samples $g \in \Stab_G(f)$ from the Haar measure and transmits $c = \pi(g^{-1}) \cdot x$. Bob evaluates $f(c) = f(x)$, recovering the inference $f(x)$ from the masked data $c$. The pointwise protocol requires only the model and representation, not a nontrivial stabilizer. Alice holds data $x \in V$ and computes $N_G(f, x)$, the level set of $g \mapsto f(\pi(g^{-1}) \cdot x)$ through the identity. She samples $g$ from this set and transmits $c = \pi(g^{-1}) \cdot x$. Bob evaluates $f(c) = f(x)$. The preimage theorem guarantees that $N_G(f,x)$ has dimension at least $\dim G - 1$ at regular points, so the scheme applies to any model regardless of whether its stabilizer is trivial.

% Requires: \usetikzlibrary{arrows.meta, positioning, calc}
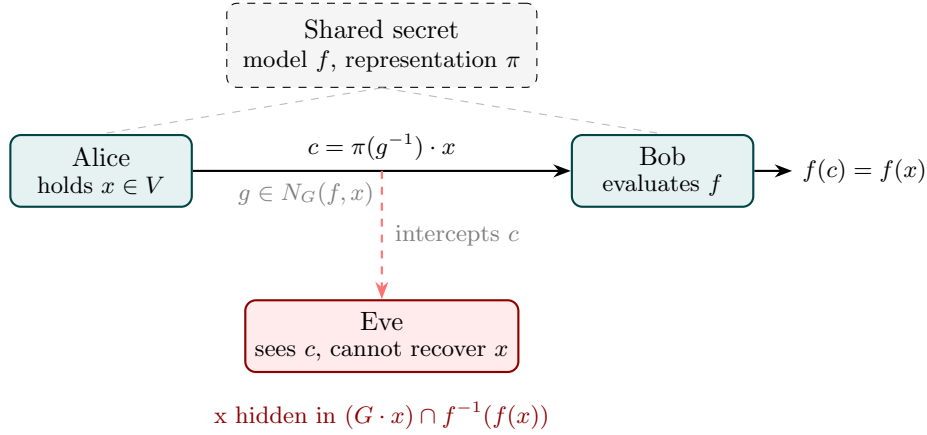
\begin{figure}[t]
\centering
\begin{tikzpicture}[
    >=Stealth,
    font=\small,
    box/.style={draw, rounded corners, minimum width=2.4cm, minimum height=0.95cm,
                align=center, thick},
    party/.style={box, fill=teal!10, draw=teal!55!black},
    adversary/.style={box, fill=red!8, draw=red!55!black},
    secret/.style={draw, rounded corners, dashed, fill=black!4, align=center,
                   minimum height=0.95cm, inner sep=6pt},
    leader/.style={dashed, gray!55},
]
% --- legitimate parties ---
\node[party] (alice) {Alice\\[-2pt]{\footnotesize holds $x \in V$}};
\node[party, right=5cm of alice] (bob) {Bob\\[-2pt]{\footnotesize evaluates $f$}};
\coordinate (mid) at ($(alice)!0.5!(bob)$);

% --- shared secret banner above the midpoint ---
\node[secret, above=1.1cm of mid] (key)
    {Shared secret\\[-1pt]{\footnotesize model $f$, representation $\pi$}};
\draw[leader] (key.south) -- (alice.north);
\draw[leader] (key.south) -- (bob.north);

% --- public channel ---
\draw[->, thick] (alice) --
    node[above, font=\footnotesize] {$c = \pi(g^{-1}) \cdot x$}
    node[below, font=\footnotesize, text=gray, pos=0.3] {$g \in N_G(f, x)$} (bob);

% --- Bob's recovered output ---
\node[right=0.5cm of bob, font=\footnotesize] (out) {$f(c) = f(x)$};
\draw[->, thick] (bob) -- (out);

% --- eavesdropper ---
\node[adversary, below=1.7cm of mid] (eve)
    {Eve\\[-2pt]{\footnotesize sees $c$, cannot recover $x$}};
\draw[->, thick, dashed, red!55] (mid) --
    node[right=1pt, font=\footnotesize, text=gray] {intercepts $c$} (eve.north);
\node[below=0.25cm of eve, font=\footnotesize, align=center, text=red!55!black]
    {x hidden in $(G\cdot x)\cap f^{-1}(f(x))$};
\end{tikzpicture}
\caption{The null fiber masking protocol. Alice masks her data $x$ using
$g \in N_G(f, x)$ and transmits $c = \pi(g^{-1}) \cdot x$. Bob recovers $f(x)$ by
evaluating $f$ on $c$. Eve sees $c$ but cannot recover $f(x)$ without the model.}
\label{fig:protocol}
\end{figure}

Let us analyze the security of this protocol. An adversary Eve intercepts $c = \pi(g^{-1}) \cdot x$ but has access to neither $g$ nor $f$. The masked input lies in the masking set $M(x) = (G \cdot x) \cap f^{-1}(f(x)) \subseteq V$. An adversary who knows the group $G$ and representation $\pi$ can recover the orbit $G \cdot c = G \cdot x$ from $c$. Knowledge of $f$ would let her narrow this to $M(x)$, but no further because every point of $M(x)$ yields the same inference, so $f$ cannot single out $x$ among them. Recovering $x$ itself requires the masking element $g$, which Alice never transmits. The security therefore rests on Eve having access to neither $f$ nor $g$, and this holds even when the model has a trivial stabilizer.

The scheme is symmetric, with the trained model $f$ as the shared secret. The pointwise scheme has greater masking power than the global scheme in general, since the null fiber $N_G(f, x) \subseteq G$ always contains $\Stab_G(f)$ and is typically much larger, so the masking set $M(x)$ it induces in $V$ is correspondingly larger than the orbit of $\Stab_G(f)$ through $x$. The cost is that Alice must compute $N_G(f, x)$ for her specific $x$, which by the algorithm of \Cref{sec:computation} requires on the order of ten model evaluations.

Even if Eve knows the task, the architecture, the training data, and the hyperparameters, and trains her own surrogate model $\hat{f}$, the null fibers of $\hat{f}$ generically differ from those of $f$. A surrogate does not even recover $M(x)$, since its level set through $c$ is a different subset of the orbit. We verify this experimentally in \Cref{sec:fingerprinting-experiments}; for the QM9 dataset, a null fiber element of one model moves a second model's prediction by as much as an unrelated random transformation does, with a ten-order-of-magnitude gap between own-model and cross-model verification. Possessing a surrogate is not equivalent to possessing the key.

The practical use case is outsourced or distributed inference on sensitive data. A hospital sends a masked scan to a cloud service running a diagnostic model, and the cloud service returns the diagnosis but never sees the scan. A sensor platform sends masked readings to a fusion center, and the fusion center runs the joint model but never accesses the raw sensor data. This is related to the goal of homomorphic encryption~\cite{gentry2009fully,gilad2016cryptonets}, which achieves inference on encrypted data at enormous computational cost and with cryptographic security guarantees. The null fiber scheme is cheaper, requiring a handful of gradient evaluations, but offers masking with a model-tied leakage profile rather than semantic security.

In the event that the model itself cannot be shared with the masking agent, which could be the case if the model is proprietary or classified, a discretization of the dataspace can be constructed and an element from the null fiber at each of the finite number of points can be computed and provided to the masking agent. Of course, if the masking agent then sends their data to the model holder, then the model holder is able to recover the data since they devised the discrete lookup table. This problem is overcome by introducing a mutually trusted third party who the model holder trusts with their model. The third party uses their knowledge of the model to create the discrete null fiber element lookup table and shares it with the masking agent. The masking agent then masks their data according to this table and sends it to the model holder. The model holder is unable to recover the data without access to the lookup table that was used but is still able to run inference. We discuss this application more in Section~\ref{sec:privacy}.

\subsubsection{Experimental demonstration}\label{sec:encryption-experiments}

We demonstrate the null fiber data masking scheme on two tasks: molecular property prediction under $\mathrm{SO}(3)$~\cite{ramakrishnan2014quantum,ruddigkeit2012enumeration} and spherical image classification under the Möbius group $\mathrm{PSL}(2, \CC)$~\cite{cohen2018spherical}. The first tests the gradient-based algorithm of \Cref{sec:computation} and establishes that exact null fiber elements can be found cheaply. The second tests whether the masking is visually effective, and whether null fiber information can be precomputed for use without model access.

For the molecular experiment, we train a four-layer MLP on the QM9 dataset to predict atomization energy from raw atomic coordinates. The model is deliberately non-equivariant, receiving only zero-padded flattened coordinates as input, so that its stabilizer is generically trivial and the null fibers are the only source of invariance. We fix a test molecule and define the orbit map $\psi(v) = f(\pi(\exp(-v)) \cdot x)$ for $v \in \mathfrak{so}(3) \cong \RR^3$, parameterized via the Rodrigues formula. One backpropagation pass yields $\nabla f(x)$, from which the Jacobian $J \in \RR^{1 \times 3}$ is assembled via~\eqref{eq:jacobian-formula}. The kernel of $J$ is a two-dimensional plane in $\mathfrak{so}(3)$, representing the infinitesimal rotations invisible to the model at $x$. Moving along a kernel direction and applying Newton correction, the algorithm finds exact null fiber elements at rotation angles of $45^\circ$ and $189^\circ$, with the prediction matching the unrotated molecule to $10^{-12}$ eV in the near case and $10^{-10}$ eV in the far case. Brute-force sampling of $10{,}000$ Haar-uniform rotations, by contrast, fails to find a single exact null fiber element: the closest sample is still $0.19$ eV off the level set because the null fiber is a two-dimensional surface in the three-dimensional group $\mathrm{SO}(3)$ and random samples almost never land on a measure-zero set. Finding $100$ distinct null fiber elements by the gradient method, spread across rotation angles from $0^\circ$ to $180^\circ$, requires $2{,}935$ total forward and backward passes. Brute-force search using $10{,}000$ samples produces $100$ approximate elements with a worst-case error of $68$ eV. The comparison is shown in \Cref{fig:gradient-vs-brute}.

\begin{figure}[t]
\centering
\includegraphics[width=\textwidth]{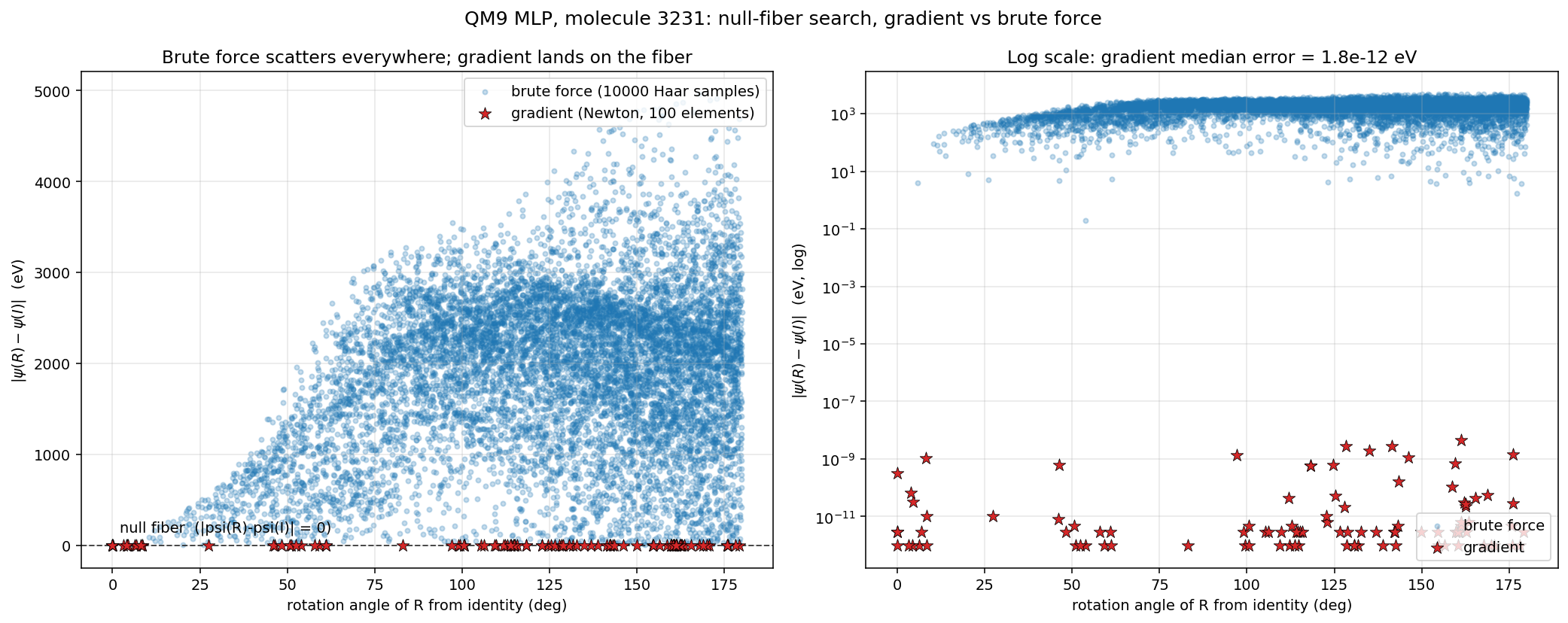}
\caption{Gradient-based null fiber computation versus brute-force sampling on a QM9 test molecule under $\mathrm{SO}(3)$. Left: prediction error $|\psi(R) - \psi(I)|$ versus rotation angle for $10{,}000$ Haar-uniform samples (blue) and $100$ gradient-computed null fiber elements (red). Right: the same data on a logarithmic scale, showing that the gradient method achieves errors with median error below $10^{-11}$~eV.}
\label{fig:gradient-vs-brute}
\end{figure}

Rotations of $\mathrm{SO}(3)$ are rigid isometries, however, and do not disguise visual content. For a steganographic demonstration~\cite{cheddad2010digital,morkel2005overview}, we need a group action that distorts the input. The Möbius group $\mathrm{PSL}(2, \CC)$ acts on the Riemann sphere $S^2 \cong \CC \cup \{\infty\}$ by $g \cdot z = (az + b)/(cz + d)$ for $g = \begin{psmallmatrix} a & b \\ c & d \end{psmallmatrix}$ with $ad - bc = 1$. These transformations are conformal but not isometric; a one-parameter family of boosts $g(t) = \begin{psmallmatrix} \cosh(t/2) & \sinh(t/2) \\ \sinh(t/2) & \cosh(t/2) \end{psmallmatrix}$ pushes content toward one pole as $|t|$ grows, producing severe area distortion that renders the image unrecognizable while remaining a smooth bijection. We train a classifier on spherical MNIST, where each digit is painted onto the sphere and represented by its spherical harmonic coefficients up to bandlimit $L = 10$, and sample $2{,}000$ Möbius transformations of the form $g = R_1 \cdot g(t) \cdot R_2$ with $R_1, R_2$ Haar-uniform in $\mathrm{SU}(2)$ and $|t| \in [0.5, 4.0]$. For each of four test digits, null fiber elements exist at large boost magnitudes $|t| \in [3.4, 4.0]$, producing heavily distorted spherical images that the classifier still labels correctly, while smaller boosts in other directions flip the classification. The direction in $\mathrm{PSL}(2, \CC)$ matters more than the magnitude. An example is shown in \Cref{fig:mobius-steganography}: the null-fiber Möbius transformation renders the digit visually unrecognizable while preserving the correct classification, whereas a transformation outside the null fiber produces a less distorted image that the classifier mislabels.

\begin{figure}[t]
\centering
\includegraphics[width=\textwidth]{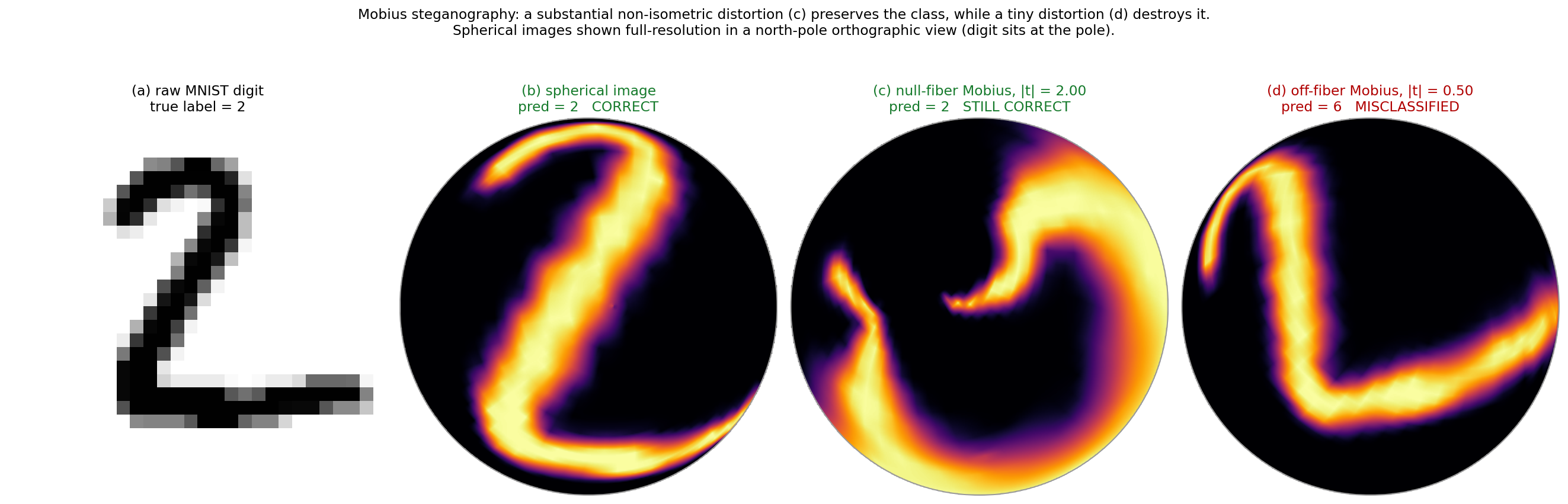}
\caption{M\"obius steganography on spherical MNIST. (a) Raw MNIST digit $2$. (b) The same digit painted on the sphere, correctly classified. (c) A null-fiber M\"obius boost at $|t| = 2$ renders the image unrecognizable, but the classifier still predicts $2$. (d) An off-fiber boost at $|t| = 0.50$ produces mild distortion but flips the prediction to $6$. Spherical images are shown in north-pole orthographic projection.}
\label{fig:mobius-steganography}
\end{figure}

The gradient-based algorithm requires model access at masking time. For settings where the masking agent cannot carry the model, we test a precomputed lookup table. Null fiber elements are computed offline for a set of reference inputs, and at masking time the agent applies the stored transformation from the nearest reference without evaluating the model. On QM9, this fails. The median prediction error after masking is $314$ eV on predictions spanning $8{,}275$ eV, with no correlation between error and distance to the nearest reference. The reason is coverage; $500$ reference molecules in an $87$-dimensional coordinate space leaves every test molecule far from its nearest reference, with median distance $7.8$ \AA, and at this distance the null fiber of the reference bears no relation to the null fiber of the test point. On spherical MNIST, the same approach works. Two hundred reference digits cover the test distribution tightly, with median nearest-reference distance $0.49$ in spherical harmonic coefficient space. The stored Möbius transformation from the nearest reference preserves the $10$-class classification $51.5\%$ of the time, compared to $11\%$ for a random transformation of the same magnitude. For digits in the closest quartile to their reference, the preservation rate reaches $84\%$. The contrast between the two datasets is explained entirely by coverage: spherical MNIST digits cluster tightly on a low-dimensional manifold, while QM9 molecules are geometrically diverse and spread across a high-dimensional space.

The most striking result emerges when the classification task is restricted to a binary channel, as shown in \Cref{fig:binary-channel}. Alice and Bob agree in advance on a pair of classes, say digits $1$ and $2$, and the secret message is which of the two the input belongs to. Alice applies a Möbius boost from the lookup table. The masked image is so distorted that the $10$-class classifier assigns it to an unrelated class, with only $28\%$ of masked digits retaining their original label. But the relative ordering of the two target logits is preserved $95.5\%$ of the time, compared to $62.5\%$ for random masking. The transformation destroys the visible classification while leaving the pairwise discrimination intact, and the secret is recoverable only by someone who knows which two logits to compare. For the harder pair of digits $3$ and $8$, which are more easily confused under distortion, the binary preservation rate is $80\%$ versus $48.5\%$ for random masking. The success of precomputed lookup is thus governed by two factors: coverage density in the input space and the complexity of the information being transmitted. High-dimensional diverse data defeats the lookup table, low-dimensional clustered data admits good coverage with modest tables, and restricting to a binary channel dramatically improves preservation because the transformation need only maintain a relative ordering rather than an absolute classification.

\begin{figure}[t]
\centering
\includegraphics[width=\textwidth]{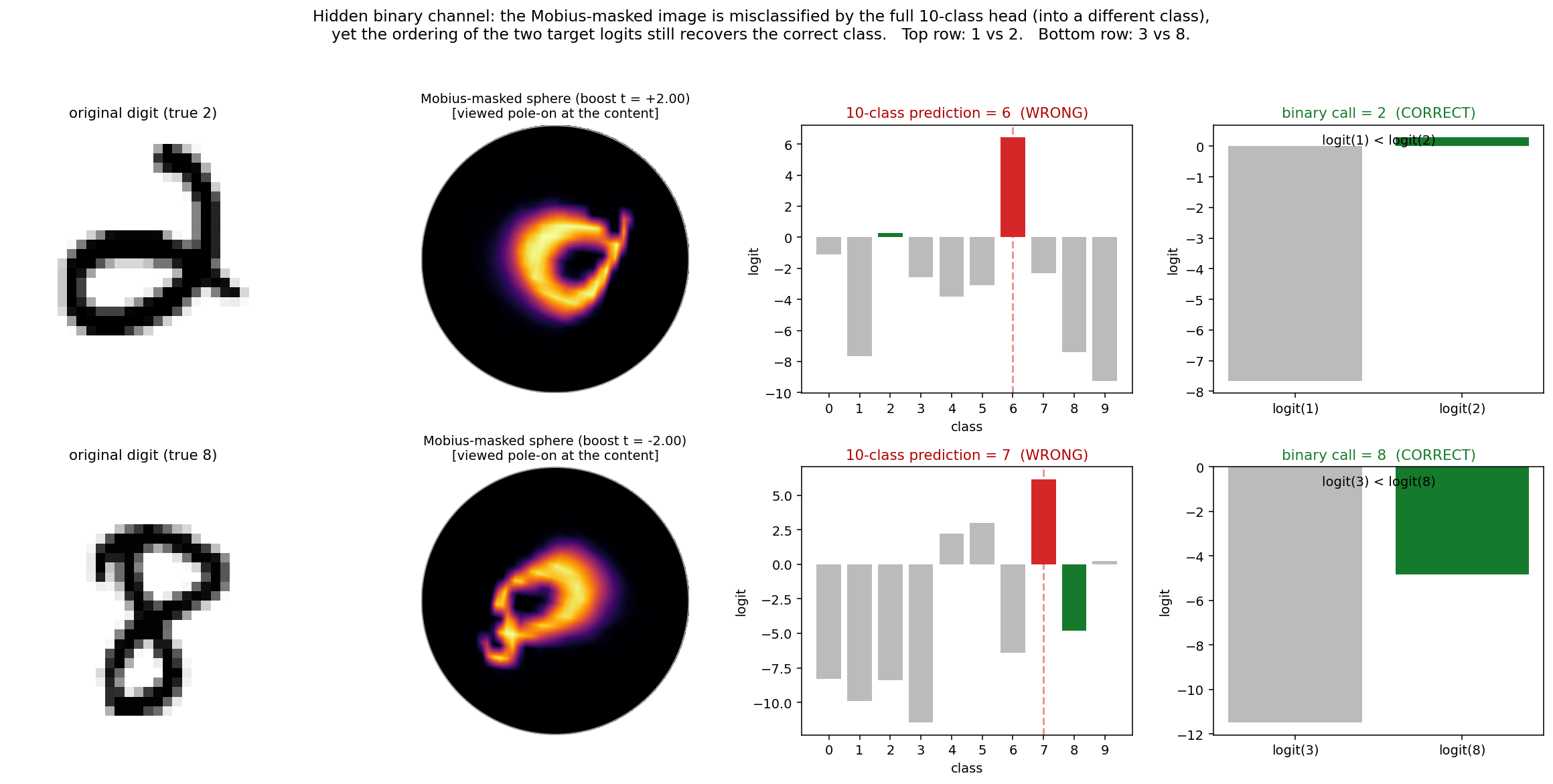}
\caption{Hidden binary channel on spherical MNIST. Top row: a true $2$ is M\"obius-masked and misclassified as $6$ by the $10$-class head, but the relative ordering logit$(1) <$ logit$(2)$ is preserved, correctly recovering the binary label. Bottom row: a true $8$ is misclassified as $7$, but logit$(3) <$ logit$(8)$ is preserved, correctly recovering the binary label. The $10$-class classification is destroyed while the pairwise channel remains intact.}
\label{fig:binary-channel}
\end{figure}

\subsection{Model fingerprinting}\label{sec:fingerprinting}

Intellectual property protection for trained models requires a way to verify that a deployed model was produced by a specific training process. A fingerprint~\cite{cao2021ipguard,lukas2019deep,peng2022fingerprinting,regazzoni2021protecting} should be intrinsic to the model, difficult to forge, and verifiable without revealing the model's parameters. The null fiber provides such a fingerprint. The stabilizer is a natural candidate but is generically trivial for expressive models, as shown in Corollary~\ref{cor:dimension}. Null fibers resolve this completely. Two independently trained models $f_1, f_2$ typically satisfy $N_G(f_1, x) \neq N_G(f_2, x)$ at every input $x$, even when $\Stab_G(f_1) = \Stab_G(f_2) = \{e\}$, because the level sets of different smooth functions are generally transverse. The collection of null fibers across inputs is a fine-grained fingerprint that distinguishes models even when the global invariant cannot.

Verification proceeds as a challenge-response protocol. The verifier selects an input $x$ and computes $y = f(x)$. The prover, who claims ownership of $f$, produces an element $g \in N_G(f, x)$ with $g \neq e$ and sends $x' = \pi(g^{-1}) \cdot x$ to the verifier. The verifier checks that $f(x') = y$ and $x' \neq x$. A party without access to $f$ cannot produce such a $g$ except by brute-force search over $G$, since determining whether a given $g$ lies in $N_G(f, x)$ requires evaluating $f$. The protocol can be repeated at multiple challenge points $x$ to increase confidence, and the prover's ability to consistently produce valid null fiber elements constitutes evidence of access to the model.

The security of this protocol rests on the cost of producing a null fiber element without gradient access to $f$. At a regular input, the null fiber of a real-valued $f$ is a smooth submanifold of codimension one. An adversary who cannot differentiate $f$ therefore cannot reach the fiber by sampling. A Haar-uniform draw lands on it with probability zero, and matching the verification tolerance $\epsilon$ requires $O(1/\epsilon)$ samples, since the $\epsilon$-tube about a codimension-one level set occupies a fraction $\propto \epsilon$ of the group. The protocol exposes no model parameters; the transcript $(x, x')$ certifies only that the prover can exhibit a nonidentity $g$ with $f(\pi(g^{-1}) \cdot x) = f(x)$, a distinct input on which $f$ agrees with its value at $x$.

\subsubsection{Experimental demonstration}\label{sec:fingerprinting-experiments}

We verify that the null fiber serves as a model-specific fingerprint by training two MLPs with identical architecture and data but different random seeds on the QM9 dataset. Both models achieve comparable test error ($335.77$ eV and $337.79$ eV), so any difference in their null fibers is attributable to the random initialization rather than a difference in model quality.

For each of $10$ test molecules, we compute $50$ null fiber elements for model $1$ using the gradient method and evaluate model $2$ on the rotated inputs. The median prediction error for model $1$ on its own null fiber elements is $1.8 \times 10^{-12}$ eV, confirming that the gradient method finds null fiber elements. The median prediction error for model $2$ on model $1$'s null fiber elements is $401$ eV, comparable to the error produced by a random rotation. The reverse experiment, computing null fiber elements for model $2$ and evaluating model $1$, produces the same pattern. A rotation that is exactly invisible to one model moves the other model's prediction by hundreds of electron volts.

We then simulate the challenge-response protocol $1{,}000$ times. The verifier selects a random test molecule $x$ and evaluates $f(x)$. The prover computes $g \in N_G(f, x)$ using the gradient method and sends $x' = \pi(g^{-1}) \cdot x$. The verifier checks whether $|f(x') - f(x)| < \epsilon$. At a threshold of $\epsilon = 0.01$ eV, model $1$'s null fiber elements pass verification against model $1$ $99.2\%$ of the time and fail against model $2$ $96.9\%$ of the time. Random Haar-uniform rotations fail against both models $100\%$ of the time. The results are symmetric; model $2$'s null fiber elements pass against model $2$ at $99.5\%$ and fail against model $1$ at $96.6\%$. The distribution of prediction errors, shown in \Cref{fig:fingerprinting}, exhibits a clean bimodal structure with a gap of roughly ten orders of magnitude between own-fiber errors clustered below $10^{-9}$ eV and cross-fiber errors clustered above $10$ eV.

\begin{figure}[t]
\centering
\includegraphics[width=\textwidth]{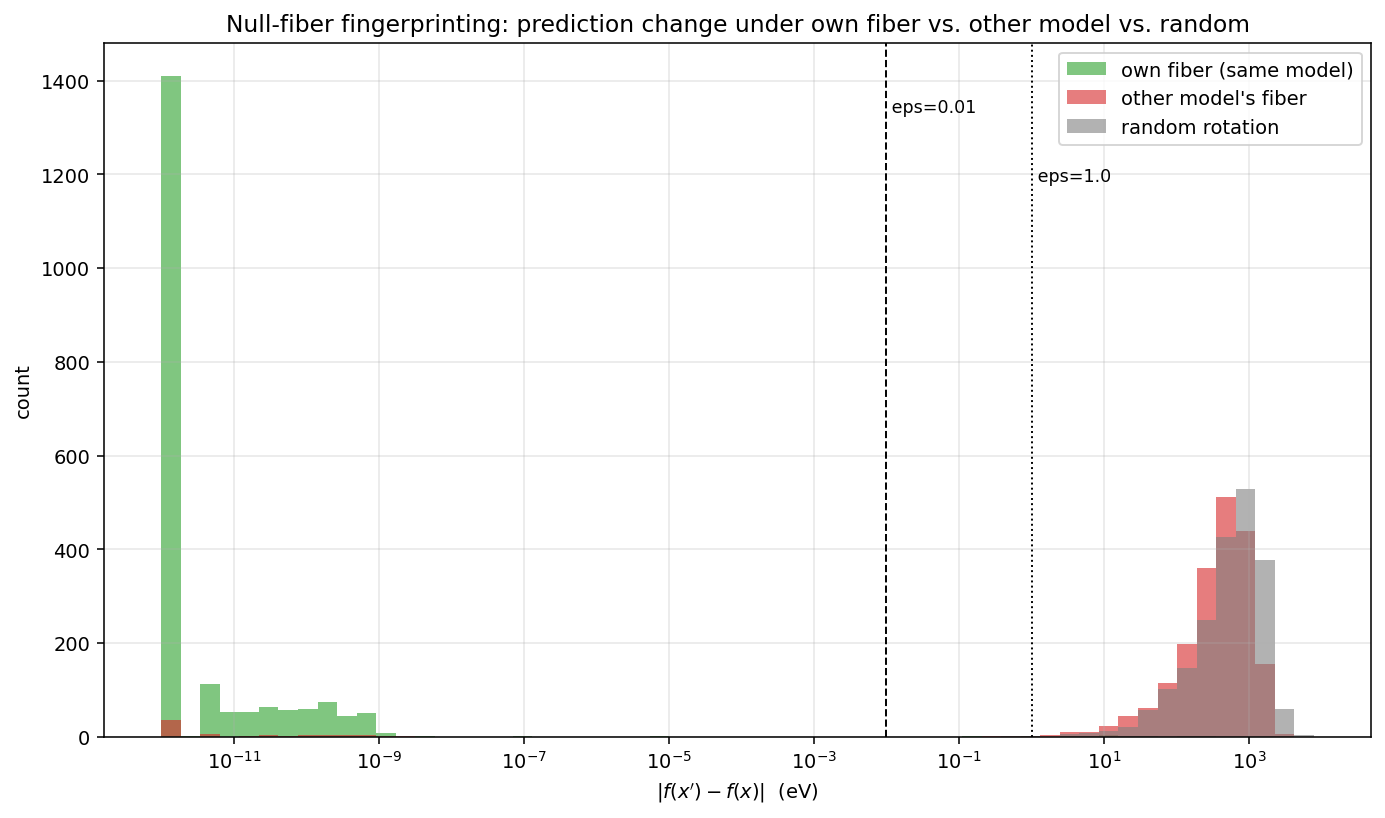}
\caption{Model fingerprinting via null fibers on QM9 regression. Distribution of prediction error $|f(x') - f(x)|$ when the rotated input $x'$ is drawn from the prover's own null fiber (left peak, below $10^{-9}$ eV), from the other model's null fiber (right peak, above $10$ eV), and from Haar-uniform random rotations (rightmost, above $100$ eV). The ten-order-of-magnitude gap between own-fiber and cross-fiber errors enables reliable verification at any reasonable threshold.}
\label{fig:fingerprinting}
\end{figure}

The small fraction of cross-model passes ($\sim 3\%$) consists of null fiber elements near the identity, where small rotation angles trivially preserve both models. Filtering to rotation angles above $20^\circ$ eliminates these. The $\sim 0.5\%$ own-model failures are Newton iterations that did not converge within the step budget, not a limitation of the framework. These two models share their architecture and training data and differ only in random initialization, which is the hardest case for distinguishing null fibers. Models with different architectures or training data will have more distinct null fibers, not less.

For classifiers, fingerprinting via the full argmax is weaker. The argmax-preserving region in $G$ is open, not a measure-zero level set, and two models trained on the same task share most of it. On spherical MNIST under $\mathrm{SO}(3)$, a rotation preserving one model's classification also preserves the other's $76\%$ of the time, compared to $11\%$ for random rotations. Under M\"obius transformations the cross-preservation drops to $69\%$, a modest improvement but still potentially too high for reliable verification. The $\mathrm{SO}(3)$ results (weaker case) are shown in \Cref{fig:fingerprinting-mnist}.

\begin{figure}[t]
\centering
\includegraphics[width=\textwidth]{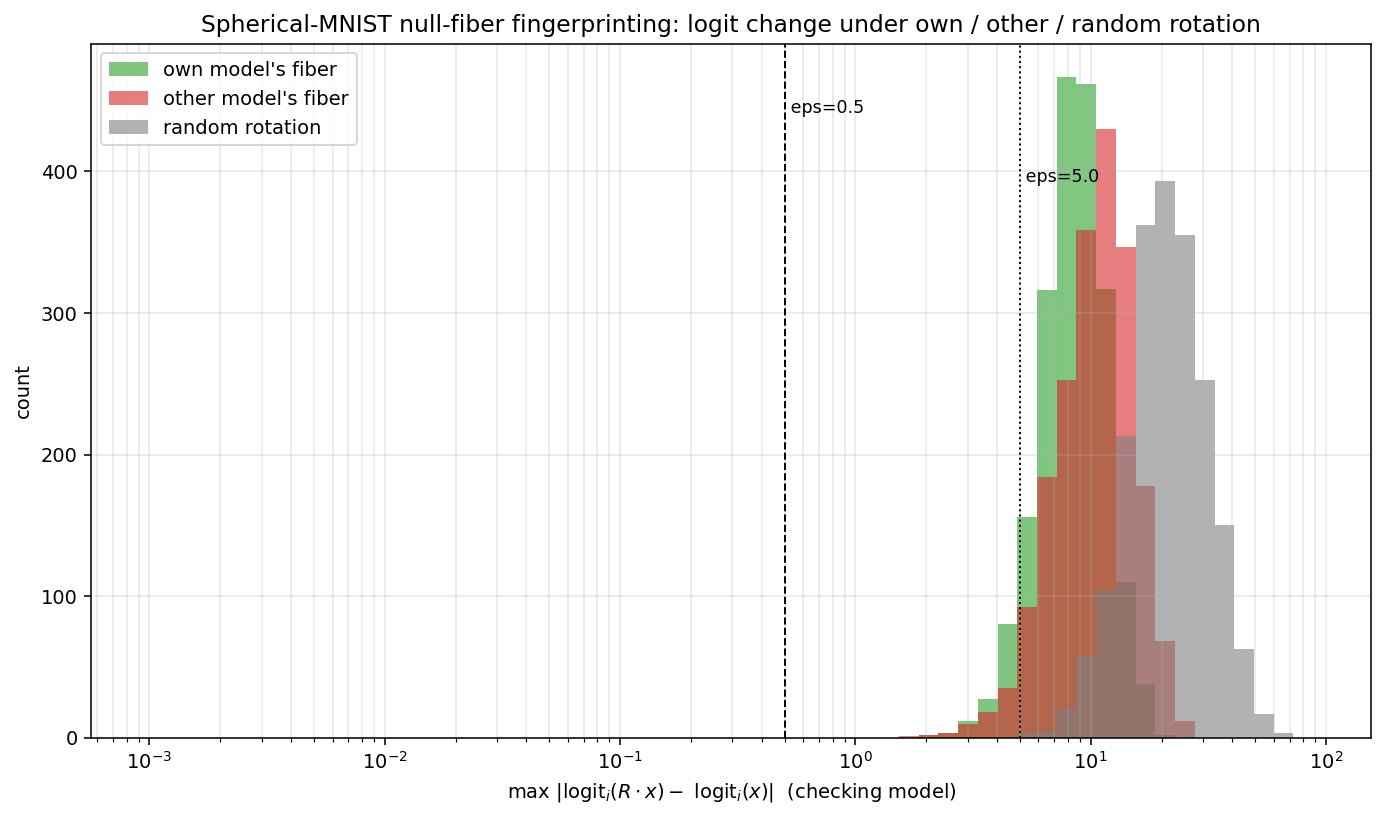}
\caption{Classification fingerprinting on spherical MNIST via the full argmax with $\mathrm{SO}(3)$ rotations. The own-model and cross-model distributions of logit deviation overlap heavily, giving a weaker fingerprint than observed with QM9 data. The argmax-preserving region is open, and two similarly-trained models share most of it.}
\label{fig:fingerprinting-mnist}
\end{figure}

Targeting a scalar quantity instead of the full argmax recovers the codimension-1 level set structure that makes regression fingerprinting work. The prover declares a logit pair $(c_1, c_2)$ and produces a group element that preserves the margin $m = \mathrm{logit}_{c_1} - \mathrm{logit}_{c_2}$. Since $m$ is a scalar, its level set is a two-dimensional surface in $\mathrm{SO}(3)$, generically unique to each model. Under $\mathrm{SO}(3)$ at a threshold of $\epsilon = 0.1$, the prover's own model passes $44.0\%$ of the time while the cross-model pass rate drops to $1.4\%$, a rejection ratio of roughly $30\times$. At $\epsilon = 1.0$, the own-model pass rate reaches $94.1\%$ against a cross-model rate of $14.7\%$. Under M\"obius transformations with a dramatic-distortion constraint $|t| > 1.5$, the cross-rejection is comparable but the own-model pass rate is lower: $19.1\%$ at $\epsilon = 0.1$ and $75.0\%$ at $\epsilon = 1.0$. The dramatic-distortion constraint forces every sampled transform far from the identity, where even the best candidate has a larger margin deviation than an SO(3) rotation at a moderate angle. The richer group does not help when the constraint prevents fine-grained search near the level set. Both pair-margin results are shown in \Cref{fig:fingerprinting-pair}.

\begin{figure}[t]
\centering
\includegraphics[width=\textwidth]{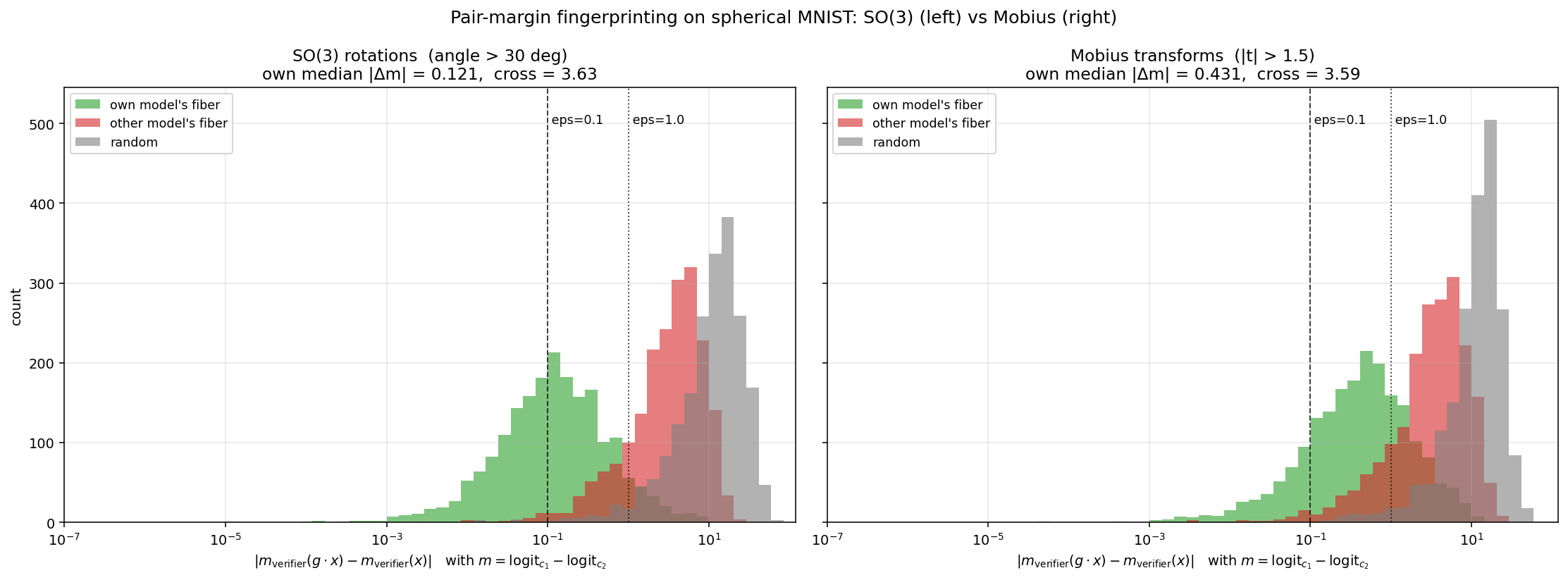}
\caption{Scalar pair-margin fingerprinting on spherical MNIST. Left: $\mathrm{SO}(3)$ with rotation angle above $30^\circ$. Right: M\"obius with $|t| > 1.5$. Targeting the scalar margin between two logits recovers a clean separation between own-model (green) and cross-model (red) distributions. $\mathrm{SO}(3)$ gives stronger separation because the moderate-angle constraint allows finer search near the level set, while the dramatic-distortion constraint on M\"obius forces all candidates far from the identity.}
\label{fig:fingerprinting-pair}
\end{figure}

\subsection{Privacy-preserving computation}\label{sec:privacy}

In federated learning and other distributed protocols~\cite{bonawitz2017practical,kairouz2021advances,li2020review,mammen2021federated,mcmahan2017communication,wen2023survey}, participants must share information derived from their local data without revealing the data itself. The null fiber provides a masking scheme with an exact utility guarantee. Each party holding local data $x$ can transmit $\pi(g^{-1}) \cdot x$ for a random $g \in N_G(f, x)$ instead of $x$ itself. Since $f(\pi(g^{-1}) \cdot x) = f(x)$ by definition of the null fiber, any downstream computation that depends only on the model's output is unaffected. The transmitted point ranges over the masking set $M(x) = (G \cdot x) \cap f^{-1}(f(x)) \subseteq V$, so what is hidden is the position of $x$ within $M(x)$. How much this protects depends on the orbit geometry; the masked variation is exactly the part of $x$ that can be moved within its orbit without changing $f(x)$, and as the experiments of \Cref{sec:privacy-experiments} show, this can range from substantial to negligible depending on the group action.

The utility guarantee is exact, since the model output is preserved identically. This contrasts with differential privacy~\cite{dwork2008differential,dwork2006calibrating}, which introduces calibrated noise that degrades utility in exchange for a quantifiable privacy bound. The two approaches are complementary rather than competing. Differential privacy protects against all possible inferences by bounding the influence of any single data point, at the cost of added noise. The null fiber masking is lossless along the fiber directions but leaves untouched the component of $x$ that $f$ depends on, and provides no formal bound on what an adversary can infer. The practical value depends on how the sensitive information in $x$ is distributed relative to the fiber geometry.

An example protocol for transmission of HIPAA protected medical data is as follows. The masking agent and owner of the sensitive data is a healthcare provider such as a hospital which needs to outsource inference of patient test data to an external organization specializing in the particular test. The masking agent does not want to send the data itself to the external organization (the model holder), so it proposes to randomly select group elements from the null fibers at each point and use them to mask the data before sending. The model holder pushes back because this procedure requires knowledge of their proprietary model. They cannot solve the problem by pre-computing a lookup table of null fiber elements at a sufficiently dense set of possible data points because the model holder then has access to the lookup table and can undo the masking upon arrival. So, the masking agent and model holder agree to consult a mutually trusted third party. The model holder agrees to share knowledge of the model with the trusted party who then computes the null fiber lookup table and passes it to the masking agent. The masking agent then masks their data using the lookup table and sends it to the model holder for inference. Since the model holder does not have access to the lookup table, they cannot recover the sensitive data, but model inference is unaffected. This procedure requires trust between the model holder and masking agent and the third party; the model holder trusts the third party with their proprietary model and the masking agent trusts the third party not to deliver the lookup table to the model holder. We illustrate this protocol in Figure~\ref{fig:hipaa-protocol}.

\begin{figure}[t]
\centering
\begin{tikzpicture}[
    >=Stealth,
    font=\small,
    box/.style={draw, rounded corners, align=center, thick,
                minimum width=3.1cm, minimum height=1.0cm},
    hosp/.style={box, fill=teal!10, draw=teal!55!black},
    model/.style={box, fill=blue!8, draw=blue!55!black},
    ttp/.style={box, fill=orange!13, draw=orange!70!black},
    slbl/.style={font=\footnotesize},
    note/.style={font=\footnotesize\itshape, text=gray},
]
% --- parties ---
\node[hosp]  (hosp)  at (0,0)      {Masking agent\\[-2pt]{\footnotesize holds $x$}};
\node[model] (model) at (9.8,0)    {Model holder\\[-2pt]{\footnotesize proprietary model $f$}};
\node[ttp]   (ttp)   at (4.9,4.0)  {Trusted third party\\[-2pt]{\footnotesize builds null-fiber table}};

% --- setup phase (upper triangle edges) ---
\draw[->, thick] (model) -- node[slbl, sloped, above, pos=0.5]{\textbf{1.}~shares model $f$} (ttp);
\draw[->, thick] (ttp)   -- node[slbl, sloped, above, pos=0.6]{\textbf{2.}~lookup table $\{g\in N_G(f,x_i)\}$} (hosp);

% --- inference phase (bottom curved edges) ---
\draw[->, thick] (hosp)  to[bend left=12]
    node[slbl, above, pos=0.5]{\textbf{3.}~masked data $c=\pi(g^{-1})\cdot x$} (model);
\draw[->, thick] (model) to[bend left=12]
    node[slbl, below, pos=0.5]{\textbf{4.}~inference $f(c)=f(x)$} (hosp);

% --- security notes ---
\node[note, above=2pt of ttp]{never sees patient data};
\node[note, below=3pt of model, align=center]{no table $\Rightarrow$ cannot recover $x$};
\end{tikzpicture}
\caption{Trusted-third-party masking protocol for outsourced inference on
HIPAA-protected data. The model holder shares its model $f$ with a trusted
third party (1), who computes a null-fiber lookup table and passes it to the
masking agent (2). The masking agent masks each record as $c=\pi(g^{-1})\cdot x$ with
$g$ drawn from the table and transmits $c$ (3); the model holder returns the
inference $f(c)=f(x)$ (4). Because the model holder never receives the lookup
table, it cannot invert the mask to recover $x$, while inference is unaffected.}
\label{fig:hipaa-protocol}
\end{figure}
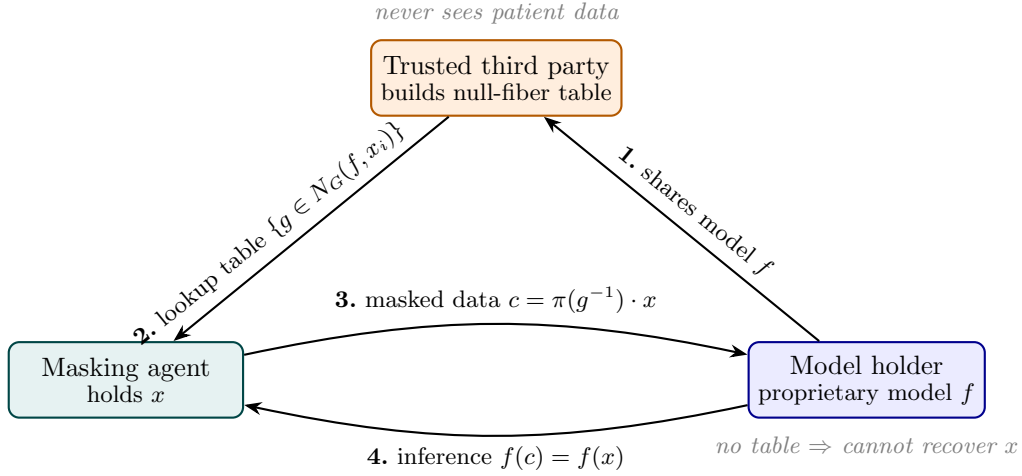

\subsubsection{Experimental demonstration}\label{sec:privacy-experiments}

The utility of null fiber masking is exact, but its privacy value depends on what information the group action moves. We demonstrate this with two experiments that produce opposite outcomes, revealing the critical role of the group action.

On spherical MNIST with Möbius transformations, the privacy is substantive. We take $200$ test digits and for each find a Möbius transformation $g$ with boost magnitude $|t| > 2$ that preserves the classifier's prediction, using the sampling approach of \Cref{sec:encryption-experiments}. All $200$ digits have such a transformation, with median $|t| = 2.50$. The classifier prediction is preserved for $100\%$ of the masked digits. The pixel-space correlation between the original and masked spherical images drops to $-0.02$, and the spherical harmonic coefficient correlation drops to $0.19$: the masked image bears no visual resemblance to the original. To quantify the privacy gain against an independent adversary, we train a separate attack MLP on unmasked spherical harmonic coefficients to predict the digit class, achieving $99.5\%$ accuracy on clean inputs. On null-fiber masked inputs, the attack accuracy drops to $54\%$. The adversary cannot reliably read the digit from the masked image, even though the classifier can.

The comparison with additive noise sharpens the point. We add Gaussian noise to the spherical harmonic coefficients, calibrated to produce the same pixel correlation destruction as the Möbius masking. At matched pixel correlation, the noise preserves the classifier prediction only $7.5\%$ of the time, compared to $100\%$ for the null fiber. The null fiber achieves privacy at zero utility cost; noise at the same level of privacy destroys the model output. A random Möbius transformation of the same magnitude, not selected from the null fiber, also destroys the classifier ($8.5\%$ preservation), confirming that the null fiber selection is essential and that the utility preservation is not a property of the Möbius group in general but of the specific fiber element chosen. The results are summarized in \Cref{fig:privacy-mobius}.

\begin{figure}[t]
\centering
\includegraphics[width=\textwidth]{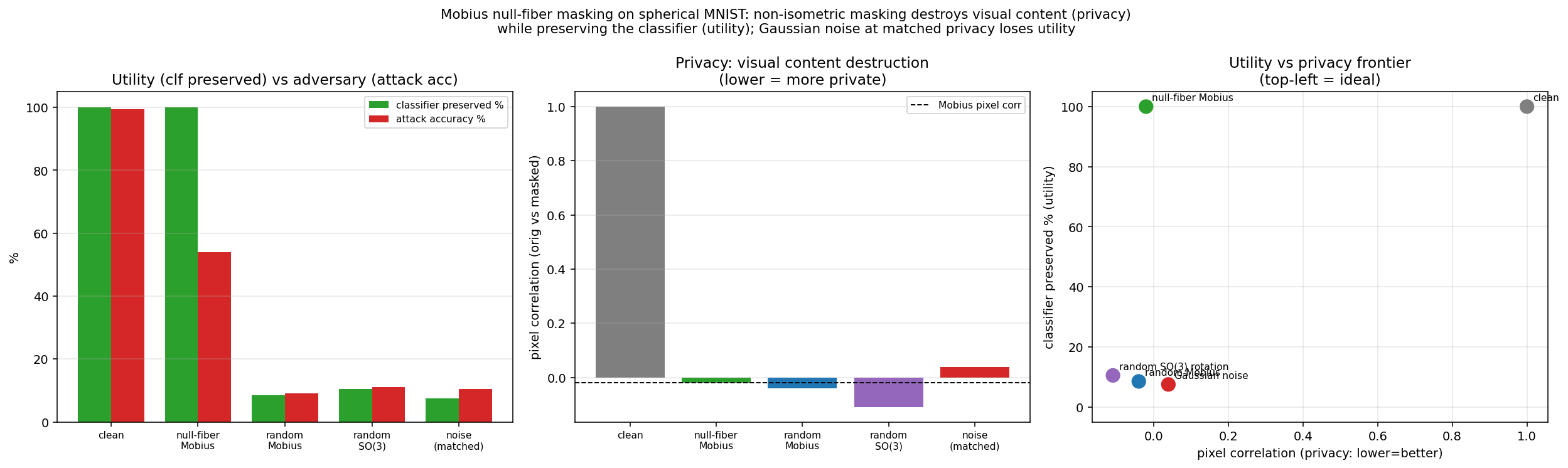}
\caption{Privacy-preserving Möbius masking on spherical MNIST. Null-fiber masking preserves the classifier prediction for $100\%$ of inputs while reducing pixel correlation to near zero and attack accuracy from $99.5\%$ to $54\%$. Gaussian noise at matched pixel correlation preserves the classifier only $7.5\%$ of the time. Random Möbius transformations of the same magnitude destroy the classifier comparably to noise.}
\label{fig:privacy-mobius}
\end{figure}

The privacy guarantee depends critically on the group action being non-isometric. On the QM9 molecular dataset under $\mathrm{SO}(3)$, null fiber masking hides only the orientation of the molecule, because rotations preserve all internal geometry; bond lengths, angles, and the full pairwise distance matrix are invariant. We verify this by computing null fiber elements for $200$ test molecules with rotation angles above $45^\circ$, confirming that the model prediction is preserved to $10^{-12}$ eV while the pairwise distance matrix correlation remains exactly $1.0$. An attack model trained to predict the radius of gyration from atomic coordinates achieves $0.032$ \AA{} error on clean inputs and $0.098$ \AA{} on masked inputs, a negligible degradation. For comparison, Gaussian noise calibrated to $(\epsilon, \delta) = (1, 10^{-5})$-differential privacy reduces the distance matrix correlation to $0.006$ and increases the radius-of-gyration attack error to $34.4$ \AA, providing genuine geometric privacy, but at the cost of an $80{,}120$ eV model prediction error on a quantity spanning roughly $8{,}000$ eV.

The two experiments together delineate the scope of null fiber privacy. When the group acts non-isometrically, as $\mathrm{PSL}(2, \CC)$ does on spherical images, the masking genuinely distorts the data while preserving the model output, providing privacy at zero utility cost. When the group acts isometrically, as $\mathrm{SO}(3)$ does on point clouds, the masking hides only the information that the group action moves, which for rigid rotations is only the global orientation.

\section{Discussion}\label{sec:discussion}

The central message of this paper is that the stabilizer, while algebraically clean, is not the right object for applications that require guaranteed information loss; the null fiber is. The stabilizer is the intersection of all null fibers, and for a generic function on a nonabelian group this intersection collapses to the identity. The Euclidean case masks this fact because the linearity of the first layer and the commutativity of $(\RR^n, +)$ force every null fiber to be the same subspace, so that the intersection loses nothing. Once either assumption is relaxed, the stabilizer shrinks and the null fiber becomes the primary carrier of geometric information about what the model discards.

The framework developed here rests on two theorems and one definition. The definition is the null fiber $N_G(f, x) = \{g \in G : f(\pi(g^{-1}) \cdot x) = f(x)\}$, which is the fiber of the orbit map over the identity. The first theorem is the preimage theorem, which guarantees that this fiber has dimension at least $\dim G - 1$ for a real-valued model at regular inputs, regardless of architecture. The second is the Fourier characterization of \Cref{thm:fourier-characterization}, which for compact groups acting on themselves reduces the stabilizer to a finite collection of row-space stabilizers of the Fourier coefficient matrices, and reduces the null fiber at a point to a single scalar equation.

The gradient-based algorithm of \Cref{sec:computation} makes the framework practical. Computing a null fiber element requires one backpropagation pass to obtain the Jacobian of the orbit map, followed by a few Newton iterations to correct the first-order approximation. The total cost is comparable to a handful of training steps, independent of the dimension of the group. The experiments confirm that this produces null fiber elements to machine precision, and that brute-force sampling cannot compete. Indeed, the null fiber is a measure-zero submanifold, and random samples almost never land on it. For settings where model access is unavailable at masking time, precomputed lookup tables offer a partial substitute whose effectiveness depends on coverage density in the input space.

The relationship to equivariant networks is complementary. A $G$-invariant model satisfies $\Stab_G(f) = G$ by construction, and the null fiber framework adds nothing. The framework becomes informative precisely when invariance fails, either because the architecture does not enforce it or because the learned weights break it. In this setting, the stabilizer measures how much global symmetry survives, and the null fibers measure how much symmetry survives at each input.

The quotient $G / \Stab_G(f)$ is a group when the stabilizer is normal, as it always is for abelian $G$, and a homogeneous space otherwise. In either case, $f$ factors through it, and the quotient describes the effective symmetry group of $f$. Null fibers carry no such quotient structure. They are submanifolds of $G$, not subgroups, and they vary from input to input. But they determine the effective degrees of freedom of $f$ at each input, which for the applications in \Cref{sec:applications} is the relevant quantity.

The experiments reveal a distinction that the theory alone does not predict: the practical value of null fiber masking depends on whether the group action is isometric. When $\mathrm{SO}(3)$ acts on molecular point clouds, rotations preserve all internal geometry, and null fiber masking hides only the global orientation. When $\mathrm{PSL}(2, \CC)$ acts on spherical images via M\"obius transformations, the action is conformal but not isometric, and the masking genuinely distorts the visual content while preserving the model output. The model output is preserved exactly along the fiber directions, but whether those directions carry sensitive information is determined by the geometry of the group action, not by the model.

The spectral characterization of \Cref{thm:fourier-characterization} applies to compact groups via the Peter--Weyl theorem. For noncompact groups such as $\mathrm{SE}(3)$, the stabilizer and null fibers remain well-defined as closed subsets of $G$, but the Fourier analysis requires the Plancherel theorem and direct integral decompositions rather than discrete sums over irreducible representations. For $G = (\RR^n, +)$ this is standard Fourier analysis and recovers the framework of~\cite{li2025null}. Extending the spectral characterization to nonabelian noncompact groups is a natural next step. A complementary boundary is the case of finite $G$ or finite $V$. The Fourier characterization of Theorem~\ref{thm:fourier-characterization} survives as finite sums over irreducible representations, but the differential-geometric content does not. The dimension bound, the regular-point hypothesis, and the gradient algorithm of Section~\ref{sec:computation} all require a Lie group with an exponential map, and for finite $G$ the null fiber is a combinatorial object rather than a positive-dimensional submanifold. The masking and fingerprinting constructions then reduce to search over a finite set, with no gradient-based shortcut.

Several problems remain open. The topology of $N_G(f, x)$ as a function of $x$ is not understood; at critical points of the orbit map the fiber may be singular, and the stratification by fiber type is a problem in singularity theory. The evolution of $\Stab_G(f_\theta)$ and $N_G(f_\theta, x)$ during training is unexplored. And the approximate null fiber $N_G^\epsilon(f, x) = \{g \in G : |f(\pi(g^{-1}) \cdot x) - f(x)| < \epsilon\}$, which has positive measure for any $\epsilon > 0$ by continuity, has no known algebraic characterization.

\bibliographystyle{plain}
\bibliography{refs}

\end{document}